\documentclass{article}

\usepackage{algorithm}
\usepackage{algorithmic}
\usepackage{microtype}
\usepackage{graphicx}
\usepackage{subfigure}
\usepackage{booktabs} 
\usepackage{amsmath}
\usepackage{amsfonts}
\usepackage{amsthm}
\usepackage{fancyhdr}
\usepackage{color}
\usepackage{eso-pic} 
\usepackage[toc,page]{appendix}
\newtheorem{remark}{Remark}

\newtheorem{theorem}{Theorem}

\newtheorem{lemma}{Lemma}

\emergencystretch=\maxdimen
\usepackage{hyperref}

\usepackage[accepted]{icml2019}

\icmltitlerunning{PrivGAN: Protecting GANs from membership inference attacks at low cost}

\begin{document}

\twocolumn[
\icmltitle{PrivGAN: Protecting GANs from membership inference attacks at low cost}




\begin{icmlauthorlist}
\icmlauthor{Sumit Mukherjee}{too}
\icmlauthor{Yixi Xu}{too}
\icmlauthor{Anusua Trivedi}{too}
\icmlauthor{Nabajyoti Patowary}{wan}
\icmlauthor{Juan Lavista Ferres}{too}
\end{icmlauthorlist}

\icmlaffiliation{too}{AI for Good Research Lab, Microsoft, Redmond, USA}
\icmlaffiliation{wan}{Microsoft, Redmond, USA}

\icmlcorrespondingauthor{Juan Lavista Ferres}{jlavista@microsoft.com}

\icmlkeywords{Membership inference, GANs}

\vskip 0.3in
]



\printAffiliationsAndNotice{}  

  \begin{abstract}
{Generative Adversarial Networks (GANs) have made releasing of synthetic images a viable approach to share data without releasing the original dataset. It has been shown that such synthetic data can be used for a variety of downstream tasks such as training classifiers that would otherwise require the original dataset to be shared. However, recent work has shown that the GAN models and their synthetically generated data can be used to infer the training set membership by an adversary who has access to the entire dataset and some auxiliary information. Current approaches to mitigate this problem (such as DPGAN~\cite{xie2018differentially}) lead to dramatically poorer generated sample quality than the original non--private GANs. Here we develop a new GAN architecture (privGAN), where the generator is trained not only to cheat the discriminator but also to defend membership inference attacks. The new mechanism is shown to empirically provide protection against this mode of attack while leading to negligible loss in downstream performances. In addition, our algorithm has been shown to explicitly prevent memorization of the training set, which explains why our protection is so effective.  The main contributions of this paper are: i) we propose a novel GAN architecture that can generate synthetic data in a privacy preserving manner with minimal hyperparameter tuning and architecture selection, ii) we provide a theoretical understanding of the optimal solution of the privGAN loss function, iii) we empirically demonstrate the effectiveness of our model against several white and black--box attacks on several benchmark datasets, iv) we empirically demonstrate on three common benchmark datasets that synthetic images generated by privGAN lead to negligible loss in downstream performance when compared against non--private GANs. While we have focused on benchmarking privGAN exclusively on image datasets, the architecture of privGAN is not exclusive to image datasets and can be easily extended to other types of datasets.}
\end{abstract}



\section{Introduction}
Much of the recent progress in machine learning and related areas has been strongly dependent on the open sharing of datasets. A recent study shows that the increase in the number of open datasets in biology has led to a strongly correlated  increase in the number of data analysis papers~\cite{piwowar2013data}. Moreover, in many specialized application areas, the development of novel algorithms is contingent upon the public availability of relevant data. While the public availability of data is essential for reproducible science, in the case of sensitive data, this poses a possible threat to the privacy of the individuals in the dataset. 

One way in which privacy of samples in a dataset can be compromised is through membership inference attacks. Membership inference attacks are adversarial attacks where the goal of the adversary is to infer whether one or more samples are a part of a dataset without having explicit access to the dataset. There has been a lot of work on developing membership inference attacks against predictive machine learning models using their outputs~\cite{shokri2017membership, long2018understanding, song2019membership, truex2019demystifying}. Much of this work has focused on exploiting information leakage due to overfitting in many machine learning models~\cite{yeom2018privacy}. These approaches have been shown to be extremely effective against common machine learning models and have given rise to machine learning methods that are specifically designed to be resistant to such attacks~\cite{nasr2018machine, jia2019memguard,li2020membership,salem2018ml}.

There has recently been a surge of interest in using synthetic data generated from generative models as a privacy preserving way to share datasets~\cite{han2018gan,yi2019generative,zheng2018detection}. While this is an appealing approach, it has been shown that generative models such as GANs are also prone to memorizing their training set~\cite{liu2018generative}. This has been exploited in several recent papers to explore the vulnerability of generative models to membership inference attacks~\cite{hayes2019logan,chen2019gan,hilprecht2019monte}. \cite{hayes2019logan} designed a white--box attack on the released discriminator of a GAN and showed that it can be almost 100\% accurate in some cases. They also designed a black--box attack, which is comparatively a lot less accurate.  \cite{hilprecht2019monte} designed Monte--Carlo attacks on the generators which are shown to have high accuracy for set membership inference (defined later) and slightly lower accuracy for instance membership inference. 

To address this vulnerability to membership inference attacks, we developed a novel GAN architecture namely priv(ate)GAN to enhance membership privacy: the synthetic data generated by the GAN trained on the training samples should be indistinguishable from those generated by the GAN trained on the other data points from the same distribution. To achieve this, an adversary is trained to attack the model, while the generator is trained to fool both the discriminator and the adversary. We empirically demonstrate the effectiveness of this architecture against attacks described in ~\cite{hayes2019logan,hilprecht2019monte} as well as an oracle attack on the discrminator that we designed. We empirically show that privGAN achieves high membership privacy while not sacrificing the sample quality, compared to the original GANs with identical architecture and hyperparameter settings. Our primary contributions in this paper are: i) proposing a novel privacy preserving GAN architecture which requires \textbf{minimal additional hyperparameter selection and no additional architecture choices}, ii) providing a theoretical analysis of the optimal solution to the GAN minimax problem and demonstrating that with large enough sample size the generative model learned with privGAN is identical to a non--private GAN, iii) empirically comparing the performance of our architecture against baselines on different membership inference attacks (both on generators and discriminators), and iv) empirically comparing the sample quality of our generated samples with different baselines.

\section{Related Works}
\subsection{Membership inference attacks against machine learning models}
A membership attacker aims to infer the membership of samples to a training set. To formally understand the adversary, let us first assume there exists a training set $X_{train}$ drawn from some data distribution $D$, a machine learning model $M$ trained on $X_{train}$ and an adversary $A$ that has access to samples $X_{adv}$ (also drawn from $D$). The adversary is assumed to have some query access to $M$ such that, given a sample $x \in X_{adv}$, it can compute some function $Q(M,x)$. The goal of the adversary is to then estimate $Pr(x \in X_{train})$ as a function of $Q(M,x)$. \cite{shokri2017membership} demonstrated one of the first membership attacks against discriminative models with only black-box access to the confidence scores for each input data sample. Their approach involved first training many 'shadow' machine learning models. Then an adversary model was trained to identify whether a sample was in the training set of any of the shadow models. This MIA approach was shown to be surprisingly effective against a variety of CNNs. Since then, there have been many such empirical membership attacks designed against machine learning models. Recently, ~\cite{yeom2018privacy, jayaraman2020revisiting}  quantified membership privacy leakage as an adversarial game and showed a connection to the generalization error of machine learning algorithms and differential privacy.

There have been several recent papers proposing successful membership inference adversaries against generative models \cite{hayes2019logan,hilprecht2019monte}. Both of these works were motivated by the close connection of information leakage to overfitting. More specifically, the generative models tend to memorize the training samples. The success of such adversaries greatly increases the risk to publish even synthetic datasets. To solve this concern, We propose privGAN, and will show the effectiveness of our method against all of these attacks in Section 6. 

\subsection{Private GANs}
A differentially private algorithm guarantees that the algorithm will yield similar outputs with high probability, no matter whether a sample is included in the training dataset or not~\cite{dwork2014algorithmic}. Existing works  \cite{xie2018differentially,jordon2018pate} propose new algorithms guaranteeing differential privacy of the discriminator (here the output is the model parameters) by the introduction of systematic noise during the model optimization steps~\cite{abadi2016deep}. This also leads to a differentially private generator as differential privacy is immune to post-processing~\cite{dwork2014algorithmic}. As a consequence, it provides a strong protection against membership inference attacks. However, the model has to sacrifice a lot in terms of sample quality and diversity of synthetic images, making them not particularly useful for practical applications. For the specific task of generating class specific images, \cite{torkzadehmahani2019dp} introduces a differentially private extension to conditional GAN~\cite{mirza2014conditional}. This can improve the downstream utility for a limited set of classification tasks. A comprehensive survey of differentially private GANs can be found in~\cite{fan2020survey}.

In addition to differential privacy based strategies, using techniques that improve generalization can also provide membership privacy benefits to generative models. Examples of such techniques are the use of dropout layers~\cite{mou2018dropout} and Wasserstein loss~\cite{arjovsky2017wasserstein}. Unlike these techniques, the privGAN architecture is explicitly designed to maximize the privacy/utility trade-off, which is quantitatively demonstrated in Section 6. 

\section{priv(ate)GANs}
In this section, we will motivate and introduce privGANs. In addition, we will provide the theoretical results for the optimal generator and discriminator.  

\subsection{The non--private GAN}

Before introducing privGANs, let us define the original GAN. In the original (non--private) GAN, the discriminator $D$ and the generator $G$ play a two-player game to minimaximize the value function $V_0(G,D)$:
\begin{align}
\label{eqn:V0}
    \begin{split}
        V_0(G,D) =& \mathbb{E}_{x\sim p_{data}(x)}[\log D(x)]+ \\
    &\mathbb{E}_{z\sim p_{z}(z)}[\log(1-D(G(z)))],
    \end{split}
\end{align}
where $p_z$ is the pre-defined input noise distribution, and $p_{data}$ is the distribution of the real data $X$. Here, the goal of the generator is to learn realistic samples that can fool the discriminator, while the goal of the discriminator is to be able to tell generator generated samples from real ones. The solution to the minimax problem leads to a generator whose generated distribution is identical to the distribution of the training dataset~\cite{goodfellow2014generative}. 

\subsection{Motivation of the privGAN architecture}
Having defined the original GAN loss function, we note that one of its major privacy risks comes from the fact that the model tends to memorize the training samples~\cite{liu2018generative,yeom2018privacy}. This leaves the original GAN vulnerable to some carefully constructed adversaries. For example, consider a situation where the trained GAN model has been open sourced, and a larger pool containing the training set is available to the public. \cite{hayes2019logan} has constructed adversarial attacks that could easily identify the training samples using the observation that the trained discriminator is more likely to identify training samples as 'real' samples than those that were not in the training set. Similarly it has been shown in~\cite{hilprecht2019monte}, that samples generated by trained generators are more similar to images in the training set than those not in the training set which can be utilized by an adversary to perform membership inference.

In the classification setting~\cite{nasr2018machine} demonstrated adversarial regularization as an effective strategy to prevent membership inference attacks. They used a built--in adversary to identify whether a sample was a part of the training set for the classification model or from a hold out set. The classification model eventually learns to make predictions that are ambigious enough to fool the adversary while trying to maintain the classification accuracy. Here, we adapt the same idea to generative models by introducing a built-in adversary which the generator must fool in addition to learning to generate realistic samples. 

\subsection{Membership inference attacks against GANs}
To motivate the built--in adversary, let us first define a generic membership inference adversary for GANs. Given a training dataset $X_{train}$, a discriminator module $D$ and a generator module $G$, the goal of a generic membership inference adversary is to learn the function $f(D,G,x)$ where: 
\begin{equation}
    Pr(x \in X_{train}) = f(D,G,x)
\end{equation}
In the case of attacks that utilize only the discriminator model (e.g. the white--box attack in~\cite{hayes2019logan}), this reduces to the form: 
\begin{equation}
    Pr(x \in X_{train}) = f(D(x))
\end{equation}
Similarly, attacks that rely only on the generator model (e.g. the Monte--Carlo attacks described in~\cite{hilprecht2019monte}) reduce to the form: 
\begin{equation}
    Pr(x \in X_{train}) = f(G,x)
\end{equation}
In the case of attacks on generators, all known attacks (to the best of our knowledge) rely on the distance of synthetic generated samples $G(z)$ with the sample of interest $x$ (using some distance metric). The underlying assumption being that if a large number of synthetic samples have a small distance to $x$, then $x$ was most likely a part of the training set.

The privGAN architecture is motivated by such an adversary.  In addition, we demonstrate through empirical results that a protection against such adversaries also protects against adversaries that target discriminators. The privGAN architecture relies on two tricks to protect against such adversaries: i) random partitioning of the training dataset to train multiple generator--discriminator pairs, ii) a built--in adversary that must be fooled, whose goal is to infer which generator generated a synthetic sample. In the following sub--section we describe the privGAN loss highlighting these different parts. 

\subsection{Some notation relevant to privGAN}
Before we introduce the mathematical formulation of privGAN, let us first list some important notation for ease of reading. Some of these terms will be defined in greater detail later on in the text:   
\begin{itemize}
    \item $N$ denotes the number of generator-discriminator pairs in privGAN
    \item The tuple ($G_i$,$D_i$) denotes the $i$th generator-disciminator pair
    \item $X_1, \cdots, X_N$ denotes the partition of data corresponding to each generator-discriminator pair
    \item $p_i$ refers to the distribution of the $X_i$
    \item $p_z$ is some pre--defined input noise distribution
    \item $D_p$ denotes the internal adversary or the 'privacy discriminator'
    \item $\lambda$ is a hyperparameter that controls the privacy--utility tradeoff in privGAN
    \item $KL(p_a||p_b)$ stands for the KL-divergence between two distributions $p_a$ and $p_b$.
    \item $JSD(p_1,\cdots,p_N)$ is the Jensen--Shannon divergence between the distributions $p_1,\cdots,p_N$. 
\end{itemize}
\subsection{The mathematical formulation of privGAN}

Given an integer valued hyperparameter $N>1$, we randomly divide the training data--set $X$ into $N$ equal sized non--overlapping subsets: $X_1, \cdots, X_N$. Each partition of the data is used to train separate generator--discriminator pairs ($G_i$,$D_i$) hence their cumulative loss is simply the summation of their individual value functions. We further introduce a built--in adversary (called the privacy discriminator) whose goal is to identify which generator generated the synthetic sample. In the case of $N=2$, this is a similar objective to that of the adversary who only utilizes the generator model of a GAN. The loss of the privacy discriminator can then be written as: \begin{equation*}
    R_p(D_p) = \mathbb{E}_{z\sim p_{z}(z)}\log[D_p^i(G_i(z))]
\end{equation*}
where, $D_p(x)=(D_p^1(x),\cdots,D_p^N(x))$ represents the probability of x to be generated by the generator $G_i$ satisfying that  $\sum\limits_{i=1}^N D_p^i(x)=1$. Hence, the complete value function $V_\lambda(\{G_i\}_{i=1}^N,\{D_i\}_{i=1}^N, D_p)$  for a privGAN is defined as: 
\begin{align*}
\begin{split}
        \sum_{i=1}^N &\left\{ \mathbb{E}_{x\sim p_{i}(x)}[\log(D_i(x))] + 
     \mathbb{E}_{z\sim p_{z}(z)}[\log(1-D_i(G_i(z)))]\right.\\
      &\left.+ \lambda \mathbb{E}_{z\sim p_{z}(z)}\log[D_p^i(G_i(z))]\right\},
\end{split}
\end{align*}
where the $p_i$ is the real data distribution of $X_i$ for $i=1,\cdots,N$, $p_z$ is the pre-defined input noise distribution, $\lambda>0$ is a hyperparameter that controls the privacy/utility trade--off. Figure \ref{Arch}A shows an illustration of the privGAN architecture when $N=2$. It is easy to see that the complete value function takes the form :
\begin{align*}
    \sum_{i=1}^N \underbrace{V_0 (G_i,D_i)}_{utility} + \lambda \underbrace{ R_p(D_p)}_{privacy}
\end{align*}
Here, the first term optimizes for 'utility' (sample quality compared to training data partition) whereas the second term optimizes for membership privacy. Accordingly, the optimization problem for privGANs is 
\begin{equation}\label{eqn:optimization}
\min_{\{G_i\}_{i=1}^N} \max_{D_p}\max_{\{D_i\}_{i=1}^N} V_\lambda(\{G_i\}_{i=1}^N,\{D_i\}_{i=1}^N, D_p). 
\end{equation}
\begin{figure*}[h!]
    \centering
    \includegraphics[scale=0.85]{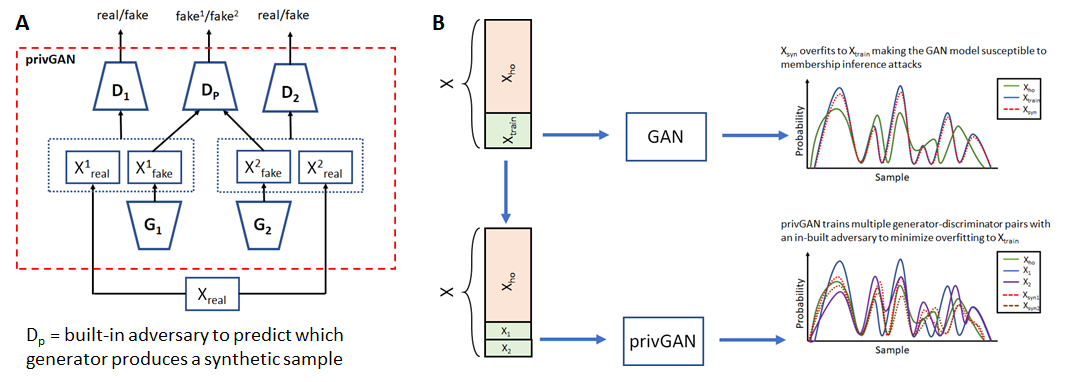}
    \caption{A) The privGAN architecture with 2 generator-discriminator pairs. B) Illustration of how privGAN provides protection against membership inference attacks by preventing memorization of the training set. }
    \label{Arch}
\end{figure*}

In section~\ref{privGAN:optimal} we will obtain the optimal solution to the above stated optimization problem. This result will then be used in section~\ref{privGAN:reg} to show that the the built--in adversary acts as a regularizer which prevents over--fitting to the training set partitions.

\subsection{Theoretical results for privGANs}
\label{privGAN:optimal}
We first provide explicit expressions for the optimal discriminators given the generators. 
\begin{theorem} \label{thm:discrim}Fixing the generators $\{G_i\}_{i=1}^N$ and the hyperparameter $\lambda>0$, the optimal discriminators of Equation \eqref{eqn:optimization} are given by:
\begin{align*}
    &D_i^{*}(x) = \frac{p_i(x)}{p_i(x)+p_{g_i}(x)},  \\
    &(D_p^i)^{*}(x) = \frac{p_{g_i}(x)}{\sum\limits_{j=1}^Np_{g_j}(x)}
\end{align*}
for $i=1,\cdots,N$, where $p_{g_j}$ is the distribution of $G_j(z)$ given $z\sim p_z$ for $j=1,\cdots,N$. 
\end{theorem}
\begin{proof}
Decompose the value function as 
\begin{align*}
    &V_\lambda(\{G_i\}_{i=1}^N,\{D_i\}_{i=1}^N, D_p)=\\ &\sum\limits_{i=1}^NV_0(G_i,D_i)+ \lambda \sum\limits_{i=1}^N\mathbb{E}_{z\sim p_{z}(z)}\log[D_p^i(G_i(z))],
\end{align*}
where $V_0$ is defined in Equation \eqref{eqn:V0}. Note that the first term $\sum\limits_{i=1}^NV_0(G_i,D_i)$ only depends on $\{D_i\}_{i=1}^N$, while the second term $\lambda \mathbb{E}_{z\sim p_{z}(z)}\log[D_p^i(G_i(z))]$ depends on $D_p$ alone. By Proposition 1 \cite{goodfellow2014generative}, $D_i^{*}(x) = \frac{p_i(x)}{p_i(x)+p_{g_i}(x)}$ maximizes $V_0(G_i,D_i)$ for $i=1,\cdots,N$. 

Note that (following~\cite{goodfellow2014generative}) 
\begin{align*}
&\sum\limits_{i=1}^N\mathbb{E}_{z\sim p_{z}(z)}\log D_p^i(G_i(z))=\sum\limits_{i=1}^N\mathbb{E}_{x\sim p_{g_i}}\log[D_p^i(x)]\\
&=\int_{x}\sum\limits_{i=1}^{N-1}p_{g_i}(x)\log D_p^i(x)+p_{g_N}(x)\log[1-\sum\limits_{j=1}^{N-1}D_p^j(x)]dx. 
\end{align*}
Then it is equivalent to solve the optimization problem $\max_{\{y_i\}_{i=1}^{N-1}}L(y_1,\cdots,y_{N-1})$, where $L(y_1,\cdots,y_{N-1})=\sum\limits_{i=1}^{N-1}a_i\log y_i+a_N\log[1-\sum\limits_{j=1}^{N-1}y_j]$ under the constraints that $\sum\limits_{i=1}^{N-1}y_i\in(0,1)$ and $y_i\in(0,1)$ for $i=1,\cdots,N-1$. It is reasonable to assume that $a_i\ge 0$, since the probability density function is always positive. Easy to verify that $L(y_1,\cdots,y_{N-1})$ is concave, given any positive $a_i$s. Note that $y_i^*=\frac{a_i}{\sum\limits_{j=1}^{N}a_j}$ for $i=1,\cdots,N-1$ solves the set of differential equations $\{\frac{\partial L}{\partial y_j}=0\}_{j=1,\cdots,N-1}$ for any positive $a_i$s. Thus it always maximizes $L(y_1,\cdots,y_{N-1})$ for any positive $a_i$s, and we complete the proof.
\end{proof}

Similar to the original GAN \cite{goodfellow2014generative}, define 
\[C_\lambda(\{G_i\}_{i=1}^N)=\max_{D_p}\max_{\{D_i\}_{i=1}^N} V_\lambda(\{G_i\}_{i=1}^N,\{D_i\}_{i=1}^N, D_p)
\]
\begin{theorem} \label{thm:optimal:g} The minimum achievable value of $C_\lambda(\{G_i\}_{i=1}^N)$ is $-N(\log 4 +\lambda \log N)$ for any positive $\lambda$. This value is achieved if and only if $p_{g_i}=p_i=p_{data}$, for $i=1,\cdots, N$. \\
\end{theorem}

\begin{proof} 
It is easy to verify that when $p_{g_i}=p_i=p_{data}$ for $i=1,\cdots, N$, $C_\lambda(\{G_i\}_{i=1}^N)$ achieves $-N(\log 4 +\lambda \log N)$. 

By its definition, $C_\lambda(\{G_i\}_{i=1}^N)$ can be re-written as: 
\begin{align*}
    \begin{split}
        &C_\lambda(\{G_i\}_{i=1}^N) = \sum_{i=1}^N\mathbb{E}_{x\sim p_{i}}[\log D_i^*(x)] +\\& \mathbb{E}_{x\sim p_{g_i}}[\log(1-D_i^*(x))]+
        \lambda\mathbb{E}_{x\sim p_{g_i}}[\log (D_p^{i*})(x)]
    \end{split}
\end{align*}

By Theorem \ref{thm:discrim} and a few algebraic manipulations, we have 
\begin{align}\label{eqn:c}
    \begin{split}
        & C_\lambda(\{G_i\}_{i=1}^N) + N(\log 4+\lambda\log N) = \\
        & \sum_{i=1}^N \left[ KL(p_i||\frac{p_i+p_{g_i}}{2}) + KL(p_{g_i}||\frac{p_i+p_{g_i}}{2}) + \right.\\
        &\left. \lambda KL(p_{g_i}||\frac{\sum_{j=1}^N p_{g_j}}{N})\right],
    \end{split}
\end{align}
 Note that the Jensen-Shannon divergence (JSD) between N distributions $p_1,...p_N$ is defined as $\frac{1}{N} \sum_{i=1}^N KL(p_i||\frac{\sum_{j=1}^N p_{j}}{N})$. Then, Equation \eqref{eqn:c} turns out to be
\begin{align*}
    \sum_{i=1}^N 2 JSD(p_i||p_{g_i})  + N \lambda JSD(p_{g_1},..,p_{g_N})\ge 0, 
\end{align*}
where the minimum is achieved if and only if $p_{g_i}=p_i$, for $i=1,\cdots, N$, and $p_{g_1}=\cdots=p_{g_N}=p_{data}$, according to the property of Jensen-Shannon divergence. Thus completing the proof.

\end{proof}

\begin{remark}
Assume that $p_i=p_{data}$ for $i=1,\cdots,N$. Given
$(\{G_i^*\}_{i=1}^N,\{D_i^*\}_{i=1}^N, D_p^*)$ - the optimal solution of Equation \eqref{eqn:optimization}, Theorems \ref{thm:discrim} and \ref{thm:optimal:g} indicates that each pair in the set \[\left\{ (G_i,D_i)_{i=1,\cdots,N}, (\sum\limits_{i=1}^NG_i/N, \sum\limits_{i=1}^ND_i/N) \right\}\] minimaximizes $V_0(G,D)$.
\end{remark}
This remark suggests that privGANs and GANs yield the same solution, when the data distribution of each partition $X_i$ is identical to that of the whole dataset $X$. This will be true, if there are infinite samples in each partition $X_i$, and $p_i$ equals to the underlying distribution, where the training samples were drawn from. 

\subsection{privGAN loss as a regularization}
\label{privGAN:reg}
In Theorem \ref{thm:optimal:g} and Remark 1, we have focused on the ideal situation where we could get access to the underlying distribution, where the training samples were drawn from. In such an ideal situation, there is no room for the membership inference attacks, thus a privGAN and a GAN yield the same solution. However, in a practical scenario, this is not true (due to unavailability of infinitely many samples), making white--box and black--box attacks against GANs effective. In the following lemma, we will demonstrate that the built--in adversary ($D_p$) serves as a regularizer that prevents the optimal generators (and hence the discriminators) from memorizing the training samples.

\begin{lemma} \label{lemma:optim}Assume that $\{(G_i^\lambda)^*\}_{i=1}^N$ minimizes $C_\lambda(\{G_i\}_{i=1}^N)$ for a fixed positive $\lambda$. Then minimizing $C_\lambda(\{G_i\}_{i=1}^N)$ is equivalent to
\begin{align}
\label{opt-reform}
\begin{split}
    &\min_{\{G_i\}_{i=1}^N}\sum_{i=1}^N  JSD(p_i||p_{g_i}) \\
    &\text{subject to:}\\
    & JSD(p_{g_1},..,p_{g_N}) \leq \delta_\lambda,
\end{split}
\end{align}
where $p_{g_i} \sim G_i(z)$ given $z \sim  p_z $ for $i=1,..,N$, and $\delta_\lambda=JSD(p_{(g_1^\lambda)^*},..,p_{(g_N^\lambda)^*})$.
\end{lemma}
\begin{proof}
Since $\lambda$ and $N$ are fixed, 
reformulate $\min C_\lambda(\{G_i\}_{i=1}^N)$ as 
\begin{align}
\label{optim-form1}
    \min_{\{G_i\}_{i=1}^N}\sum_{i=1}^N  JSD(p_i||p_{g_i})  + \frac{N \lambda}{2} JSD(p_{g_1},..,p_{g_N}).
\end{align}

Assume that $\{(G_i^\lambda)^+\}_{i=1}^N$ solves Equation \eqref{opt-reform}. It also minimizes $C_\lambda(\{G_i\}_{i=1}^N)$, since $\sum\limits_{i=1}^N JSD(p_i||p_{(g_i^\lambda)^+})+\frac{N\lambda}{2} JSD(p_{(g_1^\lambda)^+},..,p_{(g_N^\lambda)^+})\le \sum\limits_{i=1}^N JSD(p_i||p_{(g_i^\lambda)^*})+\frac{N\lambda}{2}\delta_\lambda= C_\lambda(\{(G_i^\lambda)^*\}_{i=1}^N)=\min C_\lambda(\{G_i\}_{i=1}^N)$.

We will then show that $\{(G_i^\lambda)^*\}_{i=1}^N$ is a solution of Equation \eqref{opt-reform}. If the above assumption is not true, then there exists $\{G_i\}_{i=1}^N$ such that $\sum\limits_{i=1}^N JSD(p_i||p_{g_i})< \sum\limits_{i=1}^N JSD(p_i||p_{(g_i^\lambda)^*})$, and $JSD(p_{g_1},..,p_{g_N}) \leq \delta_\lambda$. Then $\sum\limits_{i=1}^N JSD(p_i||p_{g_i})+\frac{N\lambda}{2} JSD(p_{g_1},..,p_{g_N})< \sum\limits_{i=1}^N JSD(p_i||p_{(g_i^\lambda)^*})+\frac{N\lambda}{2}\delta_\lambda= C_\lambda(\{(G_i^\lambda)^*\}_{i=1}^N)=\min C_\lambda(\{G_i\}_{i=1}^N)$. This contradicts the assumption that $\{(G_i^\lambda)^*\}_{i=1}^N$ minimizes $C_\lambda(\{G_i\}_{i=1}^N)$. This completes the proof. 
\end{proof}
Theorem \ref{thm:optimal:g} and Lemma \ref{lemma:optim} provide an intuitive understanding of the properties of the optimal generator distributions. More specifically, it has been shown that the cost function $C_\lambda(\{G_i\}_{i=1}^N)$ reduces to a trade off between the distance of generator distributions and their corresponding data split, and their distance to the other generator distributions. On the one hand, the privacy discriminator can be seen as a regularization to prevent generators from memorizing their corresponding data split.  On the other hand, a privGAN will yield the same solution as a non--private GAN when there is low risk of memorization, as suggested by Remark 1. This interesting observation further demonstrates the effectiveness of our protection, as the success of all kinds of membership inference attacks heavily rely on the extent of training set memorization. In addition, the reformulation of the optimization problem for the generators (seen in Lemma 1) provides a more explicit way to bound the distance between the generator distributions, which can be explored in future work to provide privacy guarantees. It should also be noted that as $\lambda$ increases, the upper bound of $JSD(p_{g_1},..,p_{g_N})$ for $\{G_i^*\}_{i=1}^N$ decreases and as $\lambda \rightarrow \infty$, $JSD(p_{g_1},..,p_{g_N}) \rightarrow 0$. Hence, while $\lambda$ and $N$ are dataset dependent quantities, by increasing $\lambda$ it should be possible to reduce $JSD(p_{g_1},..,p_{g_N})$ (for the optimal solution) to a user-desired value for most datasets.

{\it Note:} While the previous mathematical formulation and theoretical results are based on the original GAN formulation~\cite{goodfellow2014generative}, the same idea of multiple generator-discriminator pairs and an internal adversary could be extended to most other GAN approaches such as WGAN~\cite{arjovsky2017wasserstein},Conditional GAN~\cite{mirza2014conditional} etc.. It must be noted though that the mathematical results will not automatically extend without modification to such cases. 

\subsection{Practical implementation of privGAN architecture}
For the purposes of practical implementation of privGAN we just duplicated the discriminator and generator architectures of the simple GAN  for all the component generator--discriminator pairs of privGAN.  The privacy discriminator is identical in architecture to other discriminators barring the activation of the final layer, which is soft--max instead of sigmoid (in other discriminators). Identical learning rates are used as in the simple GAN. Hence, the only additional hyperparameters in privGAN are $\lambda$ and the number of generator--discriminator pairs $N$. It should be noted that in the case of practical implementation, $N$ lies in a bounded range $[2,\frac{|X_{train}|}{batchSize}]$ for a fixed batch size.
\begin{figure*}[h!]
    \centering
    \includegraphics[scale=0.75]{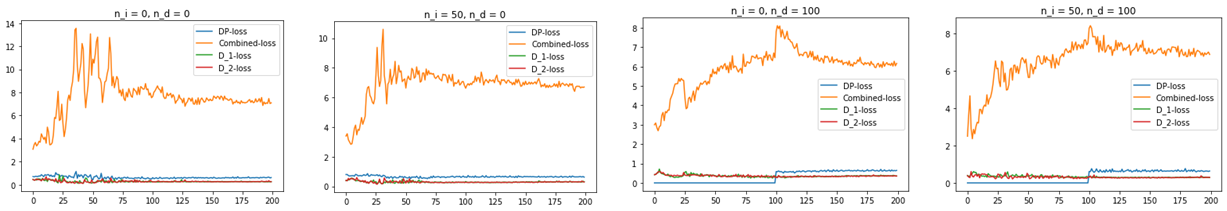}
    \caption{Comparing the loss convergence for different model training hyperparameters of privGAN on the MNIST dataset with $\lambda=1$. }
    \label{mnist-train}
\end{figure*}

\subsection{Training privGAN}
The overall training algorithm for privGAN can be seen in Algorithm~\ref{alg:train}. While an alternating minimization strategy seems like a reasonable choice for training the privGAN, there are several practical tricks that can accelerate the convergence. The first trick is to set up good initial weights for the privacy discriminator ($D_p$). We first train a neural network to distinguish the different partitions of the training data (corresponding to each generator) for a small number of epochs $n_i$ (here we used $50$). Then initialize $D_p$ with this neural network. The second trick is to fix $D_p$ for the first $n_d$ epochs (here we used $100$) after the initialization, while only allowing the generator-discriminator pairs to train. These two strategies have been shown to accelerate the convergence. In addition, the learning curves for different combinations of $n_i$ and $n_d$ are presented in Figure~\ref{mnist-train}. Setting $n_d=0$ leads to big initial transients in the combined loss which eventually subsides. A possible explanation is that it is too hard for the generators to beat both discriminators and the privacy discriminator at the very beginning. The effect of setting $n_i=0$ is less dramatic when $n_d=100$. However, the convergence of the combined loss could be accelerated by setting $n_i=50$, when $n_d = 0$. Note that $D_p$ will be initialized with random weights by setting $n_i=0$. Last but not least, the relative values of the various losses after $200$ epochs are quite stable to the choice of $n_i$ and $n_d$, as shown in Figure~\ref{mnist-train}. Note - $n_i=50$ and $n_d=100$ are used in all experiments performed in the following sections. 

\begin{algorithm}[!htb]
   \caption{Training privGAN}
   \label{alg:train}
\begin{algorithmic}
   \STATE {\bfseries Input:} Dataset $X$ with $m$ samples, array of generators $G$, array of discriminators $D$, privacy discriminator $D_p$, privacy weight $\lambda$, number of epochs $n$, initial $D_p$ epochs $n_i$, co-training delay $n_d$
   \STATE $N \gets$ number of discriminators/generators
   \STATE Divide $X$ into $N$ equal parts $\{X_1,..X_N\}$
   \STATE $y^p \gets$ partition index of each data point
   \FOR {$i = 1$ {\bfseries to} $n_i$}
   \STATE Train one epoch of $D_p$ with $(X,y^p)$
   \ENDFOR
   \FOR {$i = 1$ {\bfseries to} $n$}
   \FOR {$j = 1$ {\bfseries to} $N$}
   \STATE $X_j^f \gets$ fake images generated with $G_j$
   \STATE $y^f_j \gets$ random numbers in $[1,N]$ excluding $j$
   \STATE $y_j \gets$ labels for fake and real images
   \STATE $\hat{X}_j$ = $\{X_j,X_j^f\}$
   \STATE Train one epoch of $D_j$ with $(\hat{X}_j,y_j)$
   \ENDFOR
   \STATE $X^f \gets \{X^f_1,..X^f_N\}$
   \STATE $y^{fp} \gets \{y^f_1,..y^f_N\}$
   \IF {$i \geq n_d$}
   \STATE Train one epoch of $D_p$ with $(X^f,y^p)$
   \ENDIF
   \STATE Train one epoch of $G$ with $(X^f,y^{fp},y^p,\lambda)$
   \ENDFOR
\end{algorithmic}
\end{algorithm}

\section{Proposed Attacks}
\label{PropAttacks}
In this section, we will introduce several state-of-the-art membership inference attacks against generative models from~\cite{hilprecht2019monte,hayes2019logan},  as well as an oracle attack of our own. Although the privGAN is specifically designed to defend against membership inference attacks targeting generators (or generated data), we also present other adversaries that attack released discriminator models for completeness. 
\subsection{White--box attack on discriminator}
The white--box attack on a simple GAN is performed as outlined in~\cite{hayes2019logan}. Briefly, the attack assumes that the adversary is in possession of the trained model along with a large data pool including the training samples. The attacker is also assumed to have the knowledge of what fraction of the dataset was used for training (say $f$) but no other information about the training set. The attack then proceeds by using the discriminator of the trained GAN to obtain a probability score for each sample in the dataset (see Algorithm~\ref{alg:wb}). The samples are then sorted in descending order of probability score and the top $f$ fraction of the scores are outputted as the likely training set. The evaluation of the white--box attack is done by calculating what fraction of the predicted training set was actually in the training set.
\begin{algorithm}
   \caption{White-box attack on GAN}
   \label{alg:wb}
\begin{algorithmic}
   \STATE {\bfseries Input:} Dataset $X$ with $m$ samples, discriminator $D$
   \FOR {$i = 1$ {\bfseries to} $m$}
   \STATE $p(X_i) \leftarrow D(X_i)$
   \ENDFOR
   \STATE return $p$
\end{algorithmic}
\end{algorithm}

Since a privGAN model has multiple generator--discriminator pairs, the previously described attack cannot be directly applied to it. However, for a successful white--box attack, each of the discriminators should score samples from the training corpus higher than those not used in training (note: the training sets are of the same size for both private and non--private GANs). Hence, we modify the previous approach by identifying a \textit{max} probability score by taking the max over the scores from all discriminators (see Algorithm~\ref{alg:wb-priv}). The rationale being that the discriminator which has trained on a particular data sample should have the largest discriminator score. We now proceed to sort the samples by each of these aggregate scores and select the top $f$ fraction samples as the predicted training set, which is similar to  Algorithm~\ref{alg:wb-priv}. We also tried taking \textit{mean} instead of \textit{max} which led to largely similar results, hence we only report the results for the \textit{max} attack here.  
\begin{algorithm}
   \caption{White-box attack on privGAN}
   \label{alg:wb-priv}
\begin{algorithmic}
   \STATE {\bfseries Input:} Dataset $X$ with $m$ samples, array of discriminators $D$
   \STATE $N \gets$ number of discriminators
   \FOR {$i = 1$ {\bfseries to} $m$}
   \STATE $\{p^1(X_i),..,p^N(X_i)\} \gets \{D_1(X_i),..,D_N(X_i)\}$
   \ENDFOR
   \STATE $p^{max} \gets max(p^1,..,p^N)$
    \STATE return $p^{max}$
\end{algorithmic}   
\end{algorithm}
\subsection{Oracle white--box attack on discriminator}
While we describe a particular white--box attack targeting discriminator models from~\cite{hayes2019logan} in the previous sub--section, this is merely a heuristic based attack and there can be many other such heuristic attacks. A detailed analysis of all possible attacks is beyond the scope of this paper, but a taxonomy of possible attacks can be found in~\cite{chen2019gan}. While the previously described white--box attack is an intuitive heuristic for a practical scenario, here we seek to identify the upper limit of membership inference accuracy of white--box attacks based solely on discriminator scores. This will enable us to better quantify the privacy loss due to GANs.

We first define the following notations: 
\begin{itemize}
\item $X_{train}$ is the training set for the GAN, and $X_{ho}$ is the holdout set
\item  $D(\cdot)$ is the discriminator
    \item $x \in X_{train}\cup X_{ho}$, is a sample which may or may not have been used to train the GAN.
    \item $X$ is a random number drawn uniformly from $X_{train}\cup X_{ho}$
    \item $\mathcal{D} = \{D(x)|x\in X_{train}\cup X_{ho}\}$, and $M$ is the size of $\mathcal{D}$. Sort the set $\mathcal{D}$ in the ascending order, and define $d_i$ as its $i$th element for $i=1,\cdots, M$. 
     \item $\mathcal{A}(D(x))\in\{0,1\}$ is an adversary to infer the membership of a sample $x$ using $D(x)$. More specifically, $\mathcal{A}(D(x))=1$ means that the adversary classifies the sample $x$ as a training sample, otherwise it is classified into the holdout set $X_{ho}$ ($\mathcal{A}(D(x))=0$).   
     \item The utility score $\Delta$ of a membership inference attack $\mathcal{A}$ at a sample $x$ is: 
\begin{align*}
    \Delta_\mathcal{A}(x)= 
    \begin{cases}
    1 ,& \text{if membership correctly identified}\\
    -1, & \text{otherwise}
    \end{cases}
\end{align*}
    \item $p_i = Pr(D(X)= d_i |X \in X_{train})$
    \item $q_i = Pr(D(X) = d_i |X \in X_{ho})$
        \item $f =Pr(X\in X_{train})= \frac{|X_{train}|}{|X_{train}| + |X_{ho}|}$
    \item $s_i = Pr(D(X) = d_i) = p_i f + q_i (1-f)$
    \item $\lambda_i = Pr(X \in X_{train} | D(X) = d_i) = \frac{p_i f}{p_i f + q_i (1-f) }$
    \item $\delta^\mathcal{A}_i = \mathcal{A}(D(d_i)) $
\end{itemize}
Note that $\{p_i,q_i\}_{i=1,\cdots,M}$ could exactly describe the probability distribution of $D(X)$ given $X\in X_{train}$ or $X\in X_{ho}$. Thus, we are making no assumption about the distribution of $D(X)$. Next, we identify the maximum expected value of $\Delta$ for any adversary with access to $p_i$, $q_i$ and $f$.

\begin{theorem}
The maximum expected value of $\Delta$ is given by:
\begin{align*}
    \max_\mathcal{A}\mathbb{E}_X[\Delta_\mathcal{A}(X)]= \sum_{i=1}^M |p_i f - q_i (1-f)|
\end{align*}
\end{theorem}

\begin{proof}
The expected value of $\Delta$ is given by: 
\begin{align}
\begin{split}
    \mathbb{E}_X[\Delta_\mathcal{A}(X)]&= 1 \times Pr(\text{correct prediction by $\mathcal{A}$}) - \\  &1 \times Pr(\text{wrong prediction by $\mathcal{A}$})
\end{split} \\
    \begin{split}
            &= \sum_{i=1}^M [\lambda_i \delta^\mathcal{A}_i + (1-\lambda_i)(1-\delta^\mathcal{A}_i) - (1-\lambda_i)\delta^\mathcal{A}_i - \\
            &\lambda_i (1-\delta^\mathcal{A}_i)] s_i
    \end{split}
\end{align}
The optimal attack strategy is given by the following Integer Programming (IP) formulation: 
\begin{equation}
\begin{aligned}
       \mathcal{A}^{*} &= \arg \max_{\mathcal{A}} \sum_{i=1}^M [\lambda_i \delta^\mathcal{A}_i + (1-\lambda_i)(1-\delta^\mathcal{A}_i) 
        \\ &- (1-\lambda_i)\delta^\mathcal{A}_i - \lambda_i (1-\delta^\mathcal{A}_i)] s_i \\
    & \text{subject to:} \\
    & \delta^\mathcal{A}_i \in \{0,1\}, \forall i \in \{1,..M\}
\end{aligned}    
\end{equation}

This can be simplified to the following IP formulation: 
\begin{equation}
\begin{aligned}
\{(\delta^\mathcal{A}_i)^{\star}\}_{i=1}^M &= \arg \max_{\{\delta^\mathcal{A}_i\}_{i=1}^M} \sum_{i=1}^M (2 \lambda_i - 1) \delta^\mathcal{A}_i s_i \\
& \text{subject to:} \\
& \delta^\mathcal{A}_i \in \{0,1\}, \forall i \in \{1,..M\}
\end{aligned}
\end{equation}

It is easy to see that this IP has an analytical solution given by: 
\begin{align}
    {\delta^\mathcal{A}}_i^{\star} = 
    \begin{cases}
    1 ,& \text{if } \lambda_i \geq \frac{1}{2}\\
    0,              & \lambda_i < \frac{1}{2}
    \end{cases}, \forall i \in \{1,..,M\}
\end{align}
We note that the condition $\lambda_i \geq \frac{1}{2}$ is equivalent to $p_i\geq \frac{1-f}{f} q_i$. 

Using $\{(\delta^\mathcal{A}_i)^{\star}\}_{i=1}^M$, we find the maximum value of $\mathbb{E}_X[\Delta_\mathcal{A}(X)]$ to be: 
\begin{equation}
\begin{aligned}
    \max\mathbb{E}_X[\Delta_\mathcal{A}(X)]&= \sum_{i=1}^M [\lambda_i \mathbf{1}_{p_i\geq \frac{1-f}{f} q_i} + (1 - \lambda_i) \mathbf{1}_{q_i <  \frac{1-f}{f} p_i} 
    \\&-  (1 - \lambda_i) \mathbf{1}_{p_i\geq \frac{1-f}{f} q_i} - \lambda_i \mathbf{1}_{q_i <  \frac{1-f}{f} p_i}] s_i \\
    & =  \sum_{i=1}^M [(2 \lambda_i - 1)\mathbf{1}_{p_i\geq \frac{1-f}{f} q_i} + \\
    &(1 - 2 \lambda_i ) \mathbf{1}_{q_i <  \frac{1-f}{f} p_i}]s_i \\
    &=   \sum_{i=1}^M |p_i f - q_i (1-f)| \label{proof_final_line}
\end{aligned}    
\end{equation}
\end{proof}
It should be noted that such an oracle attack has a $\mathbb{E}_X[\Delta_\mathcal{A}(X)]_{max} = |2 f - 1| \geq 0$ even when $P = Q$.
We further show in Lemma~\ref{LemmaTVD} that for $f=\frac{1}{2}$, $\mathbb{E}_X[\Delta_\mathcal{A}(X)]_{max} = TVD(P,Q)$, where TVD stands for the Total Variation Distance between the distributions $P$ and $Q$~\cite{dudley2010distances}. 

\begin{lemma}
\label{LemmaTVD}
The maximum expected value of $\Delta$ is for $f=\frac{1}{2}$ is equal to the Total Variation Distance between probability distributions $P$ and $Q$, where $P$ is the distribution of discriminator scores given $x \in X_{test}$ while where $Q$ is the distribution of discriminator scores given $x \in X_{ho}$.
\end{lemma}

\begin{proof}
In the case of $f = \frac{1}{2}$, the condition $\lambda_i \geq \frac{1}{2}$ is equivalent to $p_i \geq q_i$. Hence, equation~\ref{proof_final_line} can be re--written as: 
\begin{align}
    \max\mathbb{E}_X[\Delta_\mathcal{A}(X)] &= \frac{1}{2} \sum_{i=1}^M |p_i - q_i|
\end{align}
This is equal to the Total Variation Distance between probability distributions $P$ and $Q$ (Q.E.D.). 
\end{proof}

Based on the previously stated observations, we use Total Variation Distance between the distribution of discriminator scores for $X_{test}$ and $X_{ho}$ to robustly quantify effectiveness of an attacker that only uses discriminator scores as described in Algorithm~\ref{alg:tvd}. Note - In the practical scenario, since the number of samples used in training are limited, we bin the discriminator scores into equally spaced bins and calculate the TVD on the resulting distributions. 

\begin{algorithm}
   \caption{TVD attack on GAN}
   \label{alg:tvd}
\begin{algorithmic}
   \STATE {\bfseries Input:} Datasets $X_{train}$ and $X_{ho}$ , discriminator $D$, number of bins $M$
   \STATE $p^i \leftarrow D(X_i), \forall i \in \{1,|X_{train}|\}$
   \STATE $q^i \leftarrow D(X_i), \forall i \in \{1,|X_{ho}|\}$
   \STATE $P \leftarrow$ probability distribution of $p^i$ with the range $[0,1]$ discretized into $M$ equal bins
   \STATE $Q \leftarrow$ probability distribution of $q^i$ with the range $[0,1]$ discretized into $M$ equal bins   
   \STATE return $TVD(P,Q)$
\end{algorithmic}
\end{algorithm}

Similar to the white--box attack described in the previous sub--section, this attack doesn't work directly on privGAN due to the presence of multiple discriminators. Hence, we modify the attack by taking the maximum TVD of the distribution of discriminator scores for $X_{test}$ and $X_{ho}$ among all discriminators (Algorithm~\ref{alg:tvd-priv}). 

\begin{algorithm}
   \caption{TVD attack on privGAN}
   \label{alg:tvd-priv}
\begin{algorithmic}
   \STATE {\bfseries Input:} Datasets $X_{train}$ and $X_{ho}$ , array of discriminators $D$, number of bins $M$
   \FOR {$j = 1$ {\bfseries to} $|D|$}
   \STATE $p^{ij} \leftarrow D_j(X_i), \forall i \in \{1,|X_{train}|\}$
   \STATE $P^j \leftarrow$ probability distribution of $p_{ij}$ with the range $[0,1]$ discretized into $M$ equal bins
   \STATE $q^{ij} \leftarrow D_j(X_i), \forall i \in \{1,|X_{ho}|\}$
   \STATE $Q^j \leftarrow$ probability distribution of $q_{ij}$ with the range $[0,1]$ discretized into $M$ equal bins
   \ENDFOR
   \STATE return $\max_j(TVD(P^j,Q^j))$
\end{algorithmic}
\end{algorithm}

\subsection{Monte--Carlo attack on generator}
While the previous two sub--sections describe attacks on GANs that rely on using the discriminator scores to infer training set membership, it has been shown in~\cite{hilprecht2019monte} that generators of GANs are also vulnerable to membership inference attacks. ~\cite{hilprecht2019monte} describes two attacks against generators, namely, instance membership inference and set membership inference. Instance membership inference is shown to work only marginally better than random chance guessing, while set membership inference was shown to be very effective in multiple datasets. Hence, in this paper we restrict ourselves to the set membership inference attack on generators only as described in~\cite{hilprecht2019monte}. 

In a set membership attack, we are given two sets of samples. All samples in one set have the desired membership, while all samples in the other set do not have said membership. The goal of the attacker is to identify which of the two sets contain samples with the desired membership. In the context of attacks against generative models, the goal is to identify the set that was part of the training set for the model. 

\begin{algorithm}
   \caption{Monte--Carlo set membership attack on GAN/privGAN}
   \label{alg:mc}
\begin{algorithmic}
   \STATE {\bfseries Input:} Datasets $X_{train}$ and $X_{ho}$ , set size $m$, synthetic dataset $X$ of size $n$, distance metric $d$, epsilon $\epsilon$ 
   \STATE $S_1 \leftarrow$, random subset of $X_{train}$ of size $m$
   \STATE $S_0 \leftarrow$, random subset of $X_{ho}$ of size $m$
   \STATE $y \leftarrow []$
   \FOR {$j = 1$ {\bfseries to} $m$}
   \STATE $f_{S_1,\epsilon} = \frac{1}{n}\sum_{i=1}^n\mathbf{1}_{x_i \in U_{S_1,\epsilon}(x_i)}$
   \STATE $f_{S_0,\epsilon} = \frac{1}{n}\sum_{i=1}^n\mathbf{1}_{x_i \in U_{S_0,\epsilon}(x_i)}$  
   \IF {$f_{S_1,\epsilon} \geq f_{S_0,\epsilon}$}
   \STATE $y[i] = 1$
   \ELSE 
   \STATE $y[i] = 0$
   \ENDIF
   \ENDFOR
   \IF {$Sum(y)>\frac{m}{2}$}
   \STATE return 1 
   \ELSIF {$Sum(y)<\frac{m}{2}$}
   \STATE return 0 
   \ELSE
   \STATE return $Bernoulli(0.5)$
   \ENDIF
\end{algorithmic}
\end{algorithm}

Algorithm~\ref{alg:mc} lays out the set membership inference algorithm described in~\cite{hilprecht2019monte}. The attack works by using the generator to generate some $n$ number of samples (or using $n$ generator generated samples if only samples are available). For each sample in the two sets whose memberships are being tested, we calculate $\frac{1}{n}\sum_{i=1}^n\mathbf{1}_{x_i \in U_{S_i,\epsilon}(x_i)}$, where $U_{\epsilon}(x) = \{x'|d(x,x')\leq \epsilon\}$. The authors find that the most effective distance is the a PCA based one. Where the top 40 principal components of the vectorized images in a held out set is first computed. To calculate the distance $d$ between two images $x$ and $x'$, the PCA transformation is used to first compute the 40 principal components of interest. An euclidean distance is then computed between these reduced dimension vectors. The authors also prescribe several heuristics for choosing $\epsilon$. The most effective heuristic is shown to be the median heuristic: 
\begin{align*}
\epsilon = \underset{1\leq i \leq 2 m}{\mathrm{median}}(\min_{1 \leq j \leq n} d(x_i,g_j))
\end{align*}
Here $g_j$ refers to the $j$th generated sample. We note that while~\cite{hilprecht2019monte} describes several minor variants of the same set membership attack (namely, different choices of distance metrics and heuristics for selecting $\epsilon$), we choose the variant that was reported to have the best results in their experiments. We note that since all three of the datasets used in~\cite{hilprecht2019monte} are also used in our paper, this selection is well motivated. 

\section{Experiment Details}
\subsection{Datasets used}
We use the following standard open datasets for our experiments: i) MNIST, ii) fashion--MNIST, iii) CIFAR-10, and iv) Labeled Faces in Wild (LFW). MNIST and fashion--MNIST are grayscale datasets of size $70,000$ ($60,000$ training samples, $10,000$ test samples). MNIST comprises of images of handwritten digits, while fashion--MNIST contains images of simple clothing items. They contain a balanced number of samples from $10$ classes of images. CIFAR-10 is a colored (RGB) dataset of everyday objects of size $60,000$ ($50,000$ training samples, $10,000$ test samples). LFW is a dataset of size $13,223$ comprising of faces of individuals. We use the grayscale version of the dataset made available through scikit--learn.

\begin{figure}[!htb]
    \centering
    \includegraphics[scale=0.32]{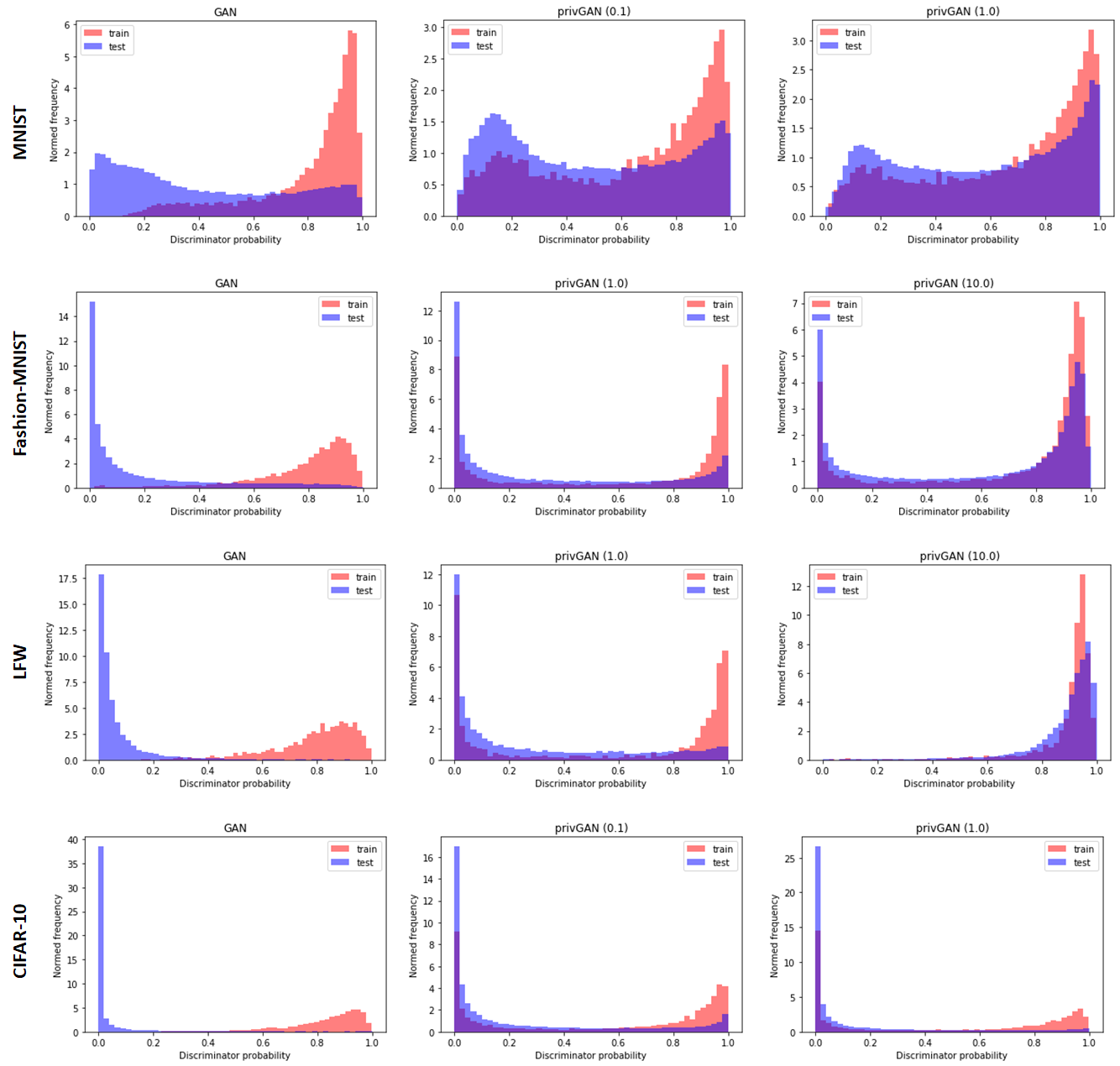}
    \caption{Comparison of predicted scores by discriminators during white box attack on privGAN and non-private GANs for the various datasets. In the case of privGAN, the scores from one randomly selected discriminator. 'Train' refers to scores of data points in the training set, while 'test' refers to scores of data points not in the training set. }
    \label{pv-loss-1}
\end{figure}

\subsection{Modeling and optimization}
For MNIST, MNIST-fashion and LFW, we use standard fully connected networks for both generators and discriminators since these are relatively simple datasets. The generator and discriminator architecture details can be found in the Supplementary Methods. Identical generator and discriminator architectures are used for both GANs and privGANs. In the comparisons with DPGANs we use an identical architecture as detailed in Supplementary Methods. While evaluating the different adversarial attacks, we trained all GAN models with an Adam~\cite{kingma2014adam} optimizer with a learning rate of $0.0002$ ($\beta = 0.5$) for $500$ epochs. While evaluating performance on downstream classification tasks, we train all GAN models with an Adam optimizer with a learning rate of $0.0002$ ($\beta = 0.5$) for $200$ epochs (except in CIFAR--10 where we train for $400$ epochs as it is a much more complicated dataset). For the classifier, we use simple CNN models (see architecture in Supplementary Methods). For the CNN models, we still used a learning rate of $0.0002$ but trained for $50$ epochs instead since the model converges quickly. In all cases we used a batch--size of 256.  

To test the efficiency of white--box  and TVD attacks, models were trained on $10\%$ of the data as in~\cite{hayes2019logan}. In the case of the Monte--Carlo attack, we first separated out the 'test set' for all datasets and used it only to compute the principal components as described in~\cite{hilprecht2019monte}. $100,000$ synthetic samples ($n$) were used in the Monte--Carlo attack. $10\%$ of the rest of the dataset was then used to train models while the model was evaluated on all the data except the held out test set. Reported numbers are averages over $10$ runs. For each run, $10\%$ of the dataset was randomly chosen to be the training set. In the case of the Monte--Carlo attack, 10 attacks were performed per run of model training and their average accuracy was taken (as in~\cite{hilprecht2019monte}). For the task of evaluating the downstream performance of GANs, a separate generative model was trained for each class of the training dataset. Here the training dataset refers to the pre--defined training set available for MNIST, MNIST--fashion and CIFAR--10.

\section{Results}
\subsection{Comparison of privacy loss under the proposed attacks}
\begin{figure*}[h!]
    \centering
    \includegraphics[scale=0.55]{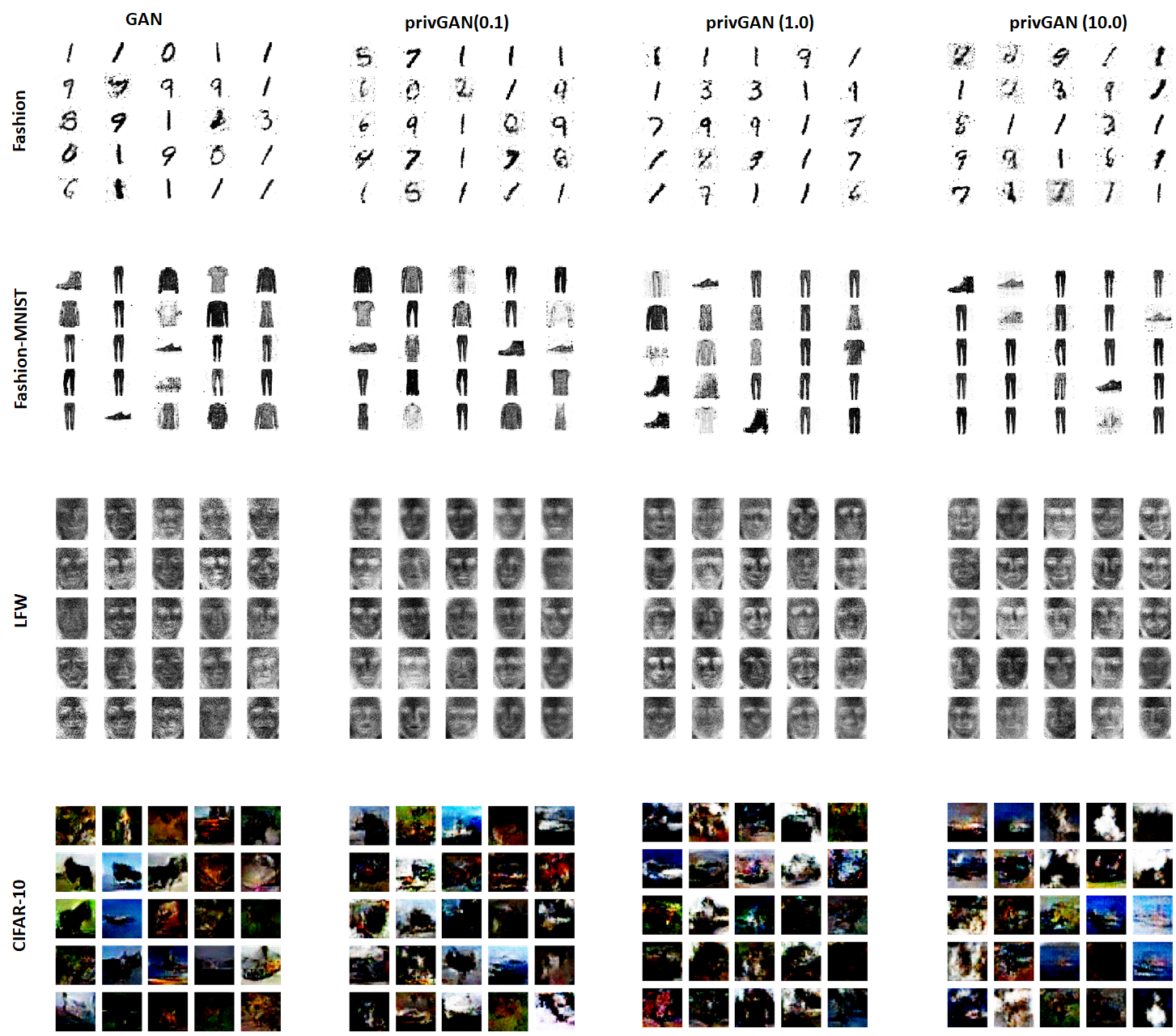}
    \caption{Comparison of images generated by non-private GAN with privGAN for different values of $\lambda$. We see a gradual drop in quality of images with increasing values of $\lambda$.}
    \label{gen-mnist}
\end{figure*}

A qualitative way to evaluate how well GANs are protected against white--box attacks is visually comparing the distribution of discriminator scores for samples in the training set with samples outside of the training set. The more similar the distributions are, the harder it is for an adversary to tell the samples apart. For a privGAN, since there are multiple discriminators, we can look at the outputs of a randomly chosen discriminator instead. In Figure~\ref{pv-loss-1} we see that the privGAN does indeed make the two distributions closer and the similarity between the distributions increases with $\lambda$. On the other hand, for a non--private GAN, the two distributions are very different which explains the high accuracy of white--box attacks in their case. 

To quantitatively compare the privacy loss of privGANs with the baselines, we performed several attacks described in section~\ref{PropAttacks}.  The first is a white--box attack as described previously. In the case of a white--box attack, since the privGAN has multiple generator/discriminator pairs, we describe a modified attack that is designed specifically for privGANs (see Algorithm~\ref{alg:wb-priv}). For each dataset, we train the GAN and privGANs (for $\lambda= 0.1,1,10$) on 10\% of the dataset. The goal of the white--box attack is to then identify the training set from the complete dataset.  Table~\ref{Table1} shows that increasing $\lambda$ generally leads to reduction in the accuracy of white--box attacks. This indicates that a privGAN becomes more resistant to membership inference attacks, as $\lambda$ becomes larger. Moreover, even for a small $\lambda=0.1$, the privGAN leads to substantial decrease in accuracy of the white--box attack when compared to the non--private GAN for all datasets. In all cases, the privGAN model corresponding to the best performing value of $\lambda$ yields comparable performance to the random chance. We also find that for two of the lower values of $\epsilon$ yielding usable images (25, 100), the white--box attack accuracy for DPGANs is similar to privGANs with $\lambda=10$ (see Supplementary Table~\ref{SuppTable1}).  To compare the effect of number of generators $N$ on privacy, we also performed the white-box attack for varying $N$ (2,4,6 and 8) with  $\lambda = 1$ (see Supplementary Table~\ref{SuppTable4}). We see that increasing $N$ generally leads to decrease in accuracy of the white-box attack, except in fashion-MNIST for $N = 8$. This may be either due to the heuristic nature of the attack or because unlike $\epsilon$ in differential privacy, the connection between $\lambda$ or $N$ to privacy is dataset dependent. We hypothesize that for certain datasets, increasing $\lambda$ or $N$ beyond certain optimal values, may cause decrease in sample quality or diversity, leading to lower membership privacy.  For the MNIST and fashion-MNIST datasets, we also compare how white--box attack accuracy varies as a function of number of epochs for privGAN ($\lambda=1$, $N=2$) and non--private GAN (see Supplementary Figures~\ref{priv-epochs-mnist}-\ref{priv-epochs-fmnist}). We find that, while increasing number of epochs increases the white-box attack accuracy for both privGAN and non-private GAN, privGAN performs significantly better than the non-private GAN at the same epoch number (against the white-box attack).
It is worth noting here that increase in privacy loss as a function of epochs is true for even differential privacy based techniques and is not unique to privGAN. 

\begin{table}[h!]
\small
\begin{center}
\begin{tabular}{c|c|c|c|c|c} 
 \hline
 \multicolumn{1}{c|}{Dataset} & \multicolumn{1}{c|}{Rand.} & \multicolumn{1}{c|}{GAN}&\multicolumn{3}{c}{privGAN}\\
 \cline{4-6}
 \multicolumn{1}{c|}{} & \multicolumn{1}{c|}{}& \multicolumn{1}{c|}{} & $\lambda=0.1$ & $\lambda=1$ & $\lambda=10$ 
 \\ \hline
MNIST  & 0.1 & 0.467 & 0.144 & 0.12 & 0.096\\ 
 f-MNIST  & 0.1 & 0.527 & 0.192 & 0.192 & 0.095\\
 LFW & 0.1 & 0.724 & 0.148 & 0.107 & 0.086\\ 
 CIFAR-10 & 0.1 & 0.723 & 0.568 & 0.424 & 0.154\\
 \hline
\end{tabular}
\end{center}
 \caption{White--box attack accuracy of various models on various datasets. For privGAN, the number represents accuracy of the 'max' attack.}
 \label{Table1}
\end{table}

While not a practical attack, the Total Variation Distance between the distribution of scores on the training and held out set provide an upper limit to the efficacy of discriminator score based white--box attacks against GANs (see Algorithm~\ref{alg:tvd}). It can be seen as an attack with an oracle adversary. Like in the previous case, this is complicated in the case of privGAN due to the presence of multiple generator--discriminator pairs. We mitigated this by taking the largest Total Variation Distance among all discriminators (see Algorithm~\ref{alg:tvd-priv}. Similar to the previous attack, we trained the GAN and privGANs (for $\lambda= 0.1,1,10$) on 10\% of the dataset. We find again that for all datasets and all three values of $\lambda$, privGAN leads to considerable reduction in Total Variation Distance (Table~\ref{Table2}). Moreover, an increase in $\lambda$ is seen to generally lead to reduction in the Total Variation Distance. Here too we note that the reduction is not a monotonic as a function of $\lambda$, possibly for the reasons stated previously. Similar to the white-box attack, we also varied the number of generators keeping $\lambda$ fixed (see Supplementary Table~\ref{SuppTable5}). We find that increasing $N$ decreases the Total Variation Distance.

\begin{table}[h!]
\small
\begin{center}
\begin{tabular}{c|c|c|c|c} 
 \hline
 \multicolumn{1}{c|}{Dataset}  & \multicolumn{1}{c|}{GAN}&\multicolumn{3}{c}{privGAN}\\
 \cline{3-5}
 \multicolumn{1}{c|}{} & \multicolumn{1}{c|}{} & $\lambda=0.1$ & $\lambda=1$ & $\lambda=10$ 
 \\ \hline
MNIST   & 0.438 & 0.31 & 0.235 & 0.048\\ 
f-MNIST  & 0.674 & 0.323 & 0.278 & 0.155\\
LFW  & 0.756 & 0.237 & 0.097 & 0.261\\ 
CIFAR-10  & 0.91 & 0.371 & 0.367 & 0.244\\
 \hline
\end{tabular}
\end{center}
 \caption{TVD attack score of various models on various datasets.}
 \label{Table2}
\end{table}

While the previous two attacks are attacks against the discriminator, the final attack is one against the trained generators and only uses synthetic generated images. Moreover, unlike the previous two attacks which were instance membership inference attacks, this is a set membership inference attack (described previously). The attack (Algorithm~\ref{alg:mc}), first described in~\cite{hilprecht2019monte}, is a Monte--Carlo attack that tries to perform set membership inference under the assumption that generated images will be more similar to the image set used to train them. As seen in the previous two attacks, privGAN outperforms GAN for all three values of $\lambda$ and set membership inference accuracy decreases as a function of $\lambda$ (Table~\ref{Table3}). Similar to the two previous attacks, we generally see a drop in attack accuracy as a function of $N$ (see Supplementary Table~\ref{SuppTable6}), with the exception of fashion-MNIST for $N=8$. 

\begin{table}[h!]
\small
\begin{center}
\begin{tabular}{c|c|c|c|c|c} 
 \hline
 \multicolumn{1}{c|}{Dataset} & \multicolumn{1}{c|}{Rand.} & \multicolumn{1}{c|}{GAN}&\multicolumn{3}{c}{privGAN}\\
 \cline{4-6}
 \multicolumn{1}{c|}{} & \multicolumn{1}{c|}{}& \multicolumn{1}{c|}{} & $\lambda=0.1$ & $\lambda=1$ & $\lambda=10$ 
 \\ \hline
MNIST  & 0.5 & 0.79 & 0.71 & 0.68 & 0.56\\ 
 f-MNIST  & 0.5 & 0.75 & 0.73 & 0.7 & 0.64\\
 LFW & 0.5 & 0.77 & 0.66 & 0.57 & 0.55\\ 
 CIFAR-10 & 0.5 & 0.62 & 0.61 & 0.56 & 0.52\\
 \hline
\end{tabular}
\end{center}
 \caption{Monte--Carlo attack accuracy of various models on various datasets.}
 \label{Table3}
\end{table}

\subsection{Comparison of downstream performance against non-private GANs}
We compare the downstream performance of privGANs against non--private GANs in two ways: i) qualitative comparison of the generated images, ii) quantitative comparison on a downstream classification task.

\begin{figure}[h!]
    \centering
    \includegraphics[scale=0.41]{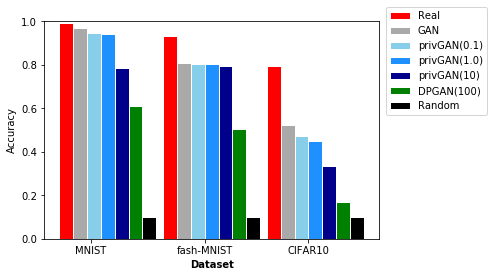}
    \caption{Comparison of test--set performance of CNN models trained real data, synthetic data generated using GAN and synthetic data generated using privGAN. Numbers in brackets indicate $\lambda$ values for privGAN and $\epsilon$ values for DPGAN.}
    \label{down-acc}
\end{figure}

For the first task we qualitatively compare the quality of images generated by privGANs with different settings of $\lambda$ to those generated by non--private GANs as seen in Figure~\ref{gen-mnist}. It is easy to see that the image quality for all three $\lambda$ values (0.1, 1, 10) are quite comparable to the images generated by non--private GANs. However, it can be seen that the image quality does decrease as we increase $\lambda$. We also see that certain classes become overrepresented as $\lambda$ increases. This will be studied in greater detail in the following section. 

To quantitatively test the downstream performance, we split the pre--defined training set for MNIST, MNIST--fashion and CIFAR by its class and trained privGANs ($\lambda = 0.1, 1, 10$) for each single class. We then generated the same number of samples per class as in the original training set to create a new synthetic training set(each image was generated by a randomly chosen generator). This training set was used to train a CNN classification model, which was then tested on the pre--defined test sets for each dataset. The baselines used for comparison were: i) CNN trained on the real training set, ii) CNN trained on a training set generated by a non--private GAN. Figure~\ref{down-acc} shows that the classification accuracy decreases, as $\lambda$ increases. However, the decrease is almost negligible with $\lambda=0.1,1$ compared to the non--private GAN for both MNIST and MNIST--fashion. Besides, a comparison with DPGAN (for $100$) shows that privGAN leads to far higher downstream utility than DPGAN. While privGANs outperforms DPGANs for all three values of $\lambda$, we note that the comparison is not fair since DPGAN is known to provide a rather conservative privacy guarantees and isn't specifically designed with membership privacy in mind. Hence, a more fair comparison is between the privacy providing mechanism used in DPGAN (gradient clipping with gaussian noise addition) with that of privGAN. To do this, we do a white--box accuracy vs downstream accuracy plot for a wide range of $\lambda$ (for privGAN) and $\epsilon$ (for DPGAN) values. (Note: This experiment was only performed on MNIST and fashion--MNIST due to computational considerations) We find that in both datasets, privGAN has higher utility for similar privacy for most of the privacy loss range as seen in Figure~\ref{priv-util}. It is also interesting to note that the privacy loss range for privGAN is somewhat smaller than DPGAN. This may indicate that simply sub-sampling the data already provides some amount of membership privacy. We tried to replicate this experiment for PATE-GAN~\cite{jordon2018pate} but were unable to generate reasonable samples, hence we don't report the results here. The implementation details and experiment results for PATE-GAN can be found in supplementary methods.

\begin{figure}[h!]
    \centering
    \includegraphics[scale=0.6]{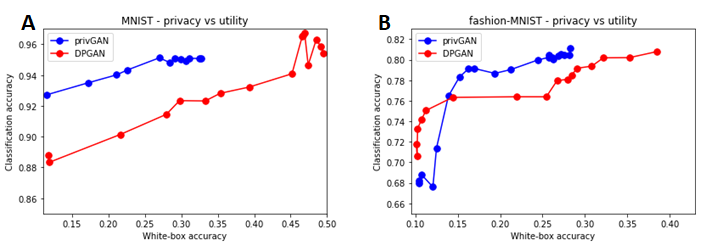}
    \caption{Comparison of utility vs privacy between privGAN and DPGAN.}
    \label{priv-util}
\end{figure}

\subsection{Effect of privGAN hyperparameters on sample quality}
\begin{figure}[h!]
    \centering
    \includegraphics[scale=0.75]{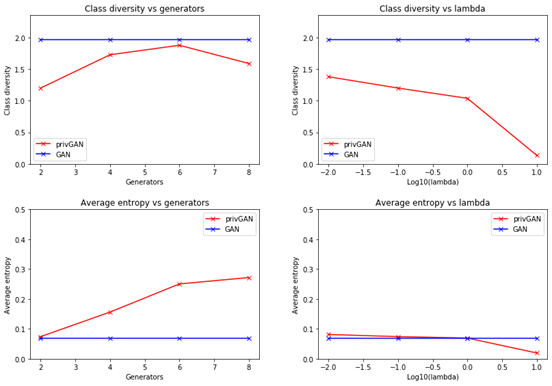}
    \caption{Comparison of the effect of hyperparameters on  average entropy and class diversity for MNIST.}
    \label{hyp-mnist}
\end{figure}

To test the effect of the hyperparameter choices on sample quality we focus on two attributes: i) unambiguity of the class of the generated images, ii) relative abundance of different classes in generated images. We measure the unambiguity of the class of the generated images using the entropy of the predicted class probabilities (using a CNN trained on real images). The average entropy of the entire dataset is then reported for different hyperparameter settings (lower average entropy represents less ambiguity of image class). The class diversity of generated images is measured by using the pre--trained CNN to first identify the most probable class per sample and then using it to calculate the relative abundance of each class in the generated dataset (scaled to sum to 1). We then calculate the entropy of the relative class abundance which we report as the class diversity (higher entropy represents larger class diversity). 

We see in Figure~\ref{hyp-mnist} that as $\lambda$ (fixing number of generators to 2) is increased, both average entropy and class diversity monotonically decrease. This implies that as $\lambda$ increases, the class ambiguity of the samples increases, while the class diversity decreases. As the number of generators is increased (fixing $\lambda=0.1$) we notice a monotonic increase of average entropy. This increase in average entropy is accompanied by an increase in class diversity (although the increase is not monotonic). This in turn implies that as the number of generators increases, the class ambiguity of samples decreases along with an increase in class diversity. Here it must be noted that as the number of generators is increased (for a fixed dataset size), the size of each data split decreases.  

Based on these results, it can be summarized that both $\lambda$ and the number of generators impact the quality of samples generated by privGAN. Since these two parameters interact, the optimal value of these hyperparameters are inter--dependent and most likely dependent on the dataset. 

\section{Practical Considerations For Data/Model Sharing}
 While the interplay between privacy and synthetic data sharing has been extensively discussed academically~\cite{bellovin2019privacy}, there is a lack of discussion on when it is appropriate to share models (and synthetic data) and when it is recommended to share just synthetic data. These release choices can dramatically affect privacy concerns as different attacks against GANs can have vastly different success rates. In our case this is further complicated because privGAN has multiple generator discriminator pairs. Here we discuss some practical considerations for synthetic data/model sharing in the case of privGAN. 

\subsection{Sharing only synthetic data}
Sharing only synthetic data is the most desirable option when possible, as it only allows for black box attacks, which are generally considered less effective than white box attacks for GANs~\cite{hayes2019logan}. While synthetic data sharing is relatively straightforward for non--private GANs, it is more complicated in the case of privGAN due to the presence of multiple generators, each of which may have lead to a slightly different generated distribution. Releasing images generated by any one generator may reduce the diversity of shared data, particularly for low values of $\lambda$. A smart sampling strategy (e.g. each image can be generated from a randomly sampled generator) from multiple generators would increase the diversity of the generated samples and would likely make it harder for adversaries to construct a black box attack. A relative privacy vs utility comparison of these approaches will be explored in future work. 

\subsection{Model sharing}
While sharing only synthetic data should be generally more preferable to model sharing from a membership privacy standpoint, in many situations sharing the model might become necessary. Moreover, it is standard practice in many academic fields to share models for the sake of reproducible research. In such circumstances, it is strongly recommended to not share the $D_p$ (privacy discriminator) module of privGAN. Since $D_p$ serves the role of an in--built adversary in privGAN, it's ability to predict which generator generated a particular synthetic sample could be utilized by an adversary to infer membership. Additionally, one may consider releasing one or a randomly selected sub--set of generator--discriminator pairs instead of all pairs to further enhance membership privacy. 

\subsection{Certifying model/data privacy}
While privGAN does not provide any privacy guarantees similar to differential privacy based techniques, some certificate of membership privacy can often be required by data/model owners as well as regulators. One could use the performance of realistic SOTA membership inference attacks against models or synthetic datasets as a way to empirically quantify membership privacy vulnerability of models. A more principled approach would be to use new privacy evaluation frameworks such as~\cite{liu2020mace}, that uses a Bayes Optimal Classifier as an optimal adversary and provides a dataset dependent and query specific membership privacy loss estimate (along with confidence intervals). This is a generalization of the Oracle attack introduced in this paper and is shown in~\cite{liu2020mace} to be related to differential privacy. Although the optimal/Oracle adversary is an impractically strong adversary, the performance estimate of such an adversary can act as a post-hoc certificate for models/synthetic datasets. A data/model owner may decide to obtain such certificates for several different queries relevant to their application, to understand overall susceptibility of the model/synthetic data.

\subsection{Hyperparameter choices}
For practical deployment purposes, we suggest users first optimize hyperparameters for a non-private GAN (including number of epochs) and then set the privGAN hyperparameters $\lambda$ and $N$. As evident from the experiment results, as well as the mathematical formulation, there is no dataset independent privacy interpretation for the hyperparameters $\lambda$ and $N$.  However, choice of these parameters affects both query specific membership privacy as well as downstream utility of the model/synthetic data. For the purposes of practical data/model sharing, we recommend performing grid search (or similar techniques) to identify optimal hyperparameters corresponding to query specific privacy (see above) and downstream utility (using the users choice of utility metric). These should allow users to identify privGAN hyperparameters that provide adequate privacy and utility for their application. Future work can be aimed at developing techniques for automatic identification of hyperparameters for a particular query and utility definition.   

\section{Conclusion}
Here we present a novel GAN framework (named privGAN) that utilizes multiple generator-discriminator pairs and a built--in adversary to prevent the model from memorizing the training set. Through a theoretical analysis of the optimal generator/discriminators, we demonstrate that the results are identical to those of a non--private GAN. We also demonstrate in the more practical scenario where the training data is a sample of the entire dataset, the privGAN loss function is equivalent to a regularization to prevent memorization of the training set. The regularization provided by privGAN could also lead to an improved learning of the data distribution, which will be the focus of future work.

To demonstrate the utility of privGAN, we focus on the application of preventing membership inference attacks against GANs. We demonstrate empirically that while non--private GANs are highly vulnerable to such attacks, privGAN provides relative protection against such attacks. While we focus on a few state--of--the--art white--box and black--box attacks in this paper, we argue that due to the intrinsic regularization effect provided by privGAN (see Section~\ref{privGAN:reg}) this would generalize to other attacks as well. We also demonstrate that compared to another popular defense against such attacks (DPGAN, PATE-GAN), privGAN minimally affects the quality of downstream samples as evidenced by the performance on downstream learning tasks such as classification. We also characterize the effect of different privGAN hyperparameters on sample quality, measured through two different metrics. 

While the major focus of the current paper has been to characterize the properties of privGAN and empirically show the protection it provides to white--box attacks, future work could focus on theoretically quantifying the membership privacy benefits due to privGAN. Specifically, it would be useful to investigate connections between the privGAN hyperparameters ($N$ and $\lambda$) and dataset memorization. The privGAN architecture could also have applications in related areas such as federated learning. Hence another direction of future work could be focused on extending privGAN to such application areas and demonstrating the benefits in practical datasets.

\bibliography{example_paper}
\bibliographystyle{icml2019}

\section*{Appendix 1: Model Architectures and Hyperparameters}
Here we outline the different layers used in the model architectures for different datasets, along with associated optimization hyperparameters. It is important to note that the same choices are made for non--private GAN, privGAN as well as DPGAN in all cases. Note that layers are in sequential order. 

\subsection*{MNIST \& MNIST--fashion}
\subsubsection*{Generator layers}
\begin{itemize}
\itemsep0em
    \item Dense(units$=512$, input size$=100$)
    \item LeakyReLU($\alpha=0.2$)
    \item Dense(units$=512$)
    \item LeakyReLU($\alpha=0.2$)
    \item Dense(units$=1024$)
    \item LeakyReLU($\alpha=0.2$)
    \item Dense(units$=784$, activation = 'tanh')
\end{itemize}

\subsubsection*{Discriminator layers}
\begin{itemize}
\itemsep0em
    \item Dense(units$=2048$)
    \item LeakyReLU($\alpha=0.2$)
    \item Dense(units$=512$)
    \item LeakyReLU($\alpha=0.2$)
    \item Dense(units$=256$)
    \item LeakyReLU($\alpha=0.2$)
    \item Dense(units$=1$, activation = 'sigmoid')
\end{itemize}

\subsubsection*{Privacy--Discriminator layers}
\begin{itemize}
\itemsep0em
    \item Dense(units$=2048$)
    \item LeakyReLU($\alpha=0.2$)
    \item Dense(units$=512$)
    \item LeakyReLU($\alpha=0.2$)
    \item Dense(units$=256$)
    \item LeakyReLU($\alpha=0.2$)
    \item Dense(units = number of generators, activation = 'softmax')
\end{itemize}

An Adam optimizer with $\beta = 0.5$ and a learning rate of $0.0002$ was used for optimization.

\subsection*{LFW}
\subsubsection*{Generator layers}
\begin{itemize}
\itemsep0em
    \item Dense(units$=512$, input size$=100$)
    \item LeakyReLU($\alpha=0.2$)
    \item Dense(units$=512$)
    \item LeakyReLU($\alpha=0.2$)
    \item Dense(units$=1024$)
    \item LeakyReLU($\alpha=0.2$)
    \item Dense(units$=2914$, activation = 'tanh')
\end{itemize}

\subsubsection*{Discriminator layers}
\begin{itemize}
\itemsep0em
    \item Dense(units$=2048$)
    \item LeakyReLU($\alpha=0.2$)
    \item Dense(units$=512$)
    \item LeakyReLU($\alpha=0.2$)
    \item Dense(units$=256$)
    \item LeakyReLU($\alpha=0.2$)
    \item Dense(units$=1$, activation = 'sigmoid')
\end{itemize}

\subsubsection*{Privacy--Discriminator layers}
\begin{itemize}
\itemsep0em
    \item Dense(units$=2048$)
    \item LeakyReLU($\alpha=0.2$)
    \item Dense(units$=512$)
    \item LeakyReLU($\alpha=0.2$)
    \item Dense(units$=256$)
    \item LeakyReLU($\alpha=0.2$)
    \item Dense(units = number of generators, activation = 'softmax')
\end{itemize}

An Adam optimizer with $\beta = 0.5$ and a learning rate of $0.0002$ was used for optimization.

\subsection*{CIFAR--10}
\subsubsection*{Generator layers}
\begin{itemize}
\itemsep0em
    \item Dense(units$=2048$, input size$=100$, target shape$= (2, 2, 512)$)
    \item Conv2DTranspose(filters$= 256$, kernel size$=5$, strides$=2$)
    \item LeakyReLU($\alpha=0.2$)
    \item Conv2DTranspose(filters$= 128$, kernel size$=5$, strides$=2$)
    \item LeakyReLU($\alpha=0.2$)
    \item Conv2DTranspose(filters$= 64$, kernel size$=5$, strides$=2$)
    \item LeakyReLU($\alpha=0.2$)
    \item Conv2DTranspose(filters$= 3$, kernel size$=5$, strides$=2$, activation = 'tanh')
\end{itemize}

\subsubsection*{Discriminator layers}
\begin{itemize}
\itemsep0em
    \item Conv2D(filters$= 64$, kernel size$=5$, strides$=2$)
    \item Reshape(target shape$= (2, 2, 512)$)
    \item Conv2D(filters$= 128$, kernel size$=5$, strides$=2$)
    \item LeakyReLU($\alpha=0.2$)
    \item Conv2D(filters$= 128$, kernel size$=5$, strides$=2$)
    \item LeakyReLU($\alpha=0.2$)
    \item Conv2D(filters$= 256$, kernel size$=5$, strides$=2$)
    \item LeakyReLU($\alpha=0.2$)
    \item Dense(units$=1$, activation = 'sigmoid')
\end{itemize}

\subsubsection*{Privacy--Discriminator layers}
\begin{itemize}
\itemsep0em
    \item Conv2D(filters$= 64$, kernel size$=5$, strides$=2$)
    \item Reshape(target shape$= (2, 2, 512)$)
    \item Conv2D(filters$= 128$, kernel size$=5$, strides$=2$)
    \item LeakyReLU($\alpha=0.2$)
    \item Conv2D(filters$= 128$, kernel size$=5$, strides$=2$)
    \item LeakyReLU($\alpha=0.2$)
    \item Conv2D(filters$= 256$, kernel size$=5$, strides$=2$)
    \item LeakyReLU($\alpha=0.2$)
    \item Dense(units = number of generators, activation = 'softmax')
\end{itemize}

An Adam optimizer with $\beta = 0.5$ and a learning rate of $0.0002$ was used for optimization.

\subsection*{CNN classifier for MNIST \& MNIST--fashion}
\begin{itemize}
\itemsep0em
    \item Conv2D(filters$= 32$, kernel size$=3$, activation = 'relu')
    \item Conv2D(filters$= 32$, kernel size$=3$, activation = 'relu')
    \item Max--pooling(pool size$=2$)
    \item Dense(units$= 128$, activation = 'relu')
    \item Dense(units$= 10$, activation = 'soft--max')
\end{itemize}
An Adam optimizer with $\beta = 0.5$ and a learning rate of $0.0002$ was used for optimization.

\subsection*{CNN classifier for CIFAR--10}
\begin{itemize}
\itemsep0em
    \item Conv2D(filters$= 32$, kernel size$=3$, activation = 'relu')
    \item Conv2D(filters$= 32$, kernel size$=3$, activation = 'relu')
    \item Max--pooling(pool size$=2$)
    \item Dropout($0.25$)
    \item Conv2D(filters$= 64$, kernel size$=3$, activation = 'relu')
    \item Conv2D(filters$= 64$, kernel size$=3$, activation = 'relu')
    \item Max--pooling(pool size$=2$)
    \item Dropout($0.25$)
    \item Dense(units$= 512$, activation = 'relu')
    \item Dropout($0.5$)
    \item Dense(units$= 10$, activation = 'soft--max')
\end{itemize}
An Adam optimizer with $\beta = 0.5$ and a learning rate of $0.0002$ was used for optimization.

\subsection*{DPGAN hypterparameters and implementation}
To make the architectures identical, we replaced the Wasserstein loss in DPGAN with the original GAN loss. In the implementation of DPGAN, there are several additional hyperparameters. In all our experiments, we have set  $\delta = 1/N$, where $N$ is the sample size of the dataset (this is considered standard practice). Furthermore, for each discriminator iteration, we performed $1$ iteration of the generator for the sake of consistency with GAN and privGAN. DPGAN was implemented using the Tensorflow Privacy package (https://github.com/tensorflow/privacy).  

\subsection*{PATE-GAN implementation}
The generator architecture for PATE-GAN was chosen to be identical to a simple GAN. The Teacher and Student discriminators were chosen to have the same architecture as a simple GAN discrminator. The number of teacher discrminators was set to $5$, following the default values found in the bitbucket repository of project (\href{https://bitbucket.org/mvdschaar/mlforhealthlabpub/src/40de054cab8f945fd0a7e6846d53605c0bcc058e/alg/pategan/}{Link}). The number of inner student and teacher discriminator iterations ($n_S$ and $n_T$) were set to $5$, also following default settings. The total number of iterations was chosen to identical to other types of GANs ($privGAN$, $DPGAN$) for different values of $\lambda \in [0.01,50]$. The white--box attacks reported here were performed against the student discrminator of PATE-GAN. We note here that using these settings we were simply unable to generate any recognizable images using PATE-GAN. Furthermore, we were not able to find any examples in published/archived literature (or even blogs) where PATE-GAN has been used for image generation. On the contrary, others have noted that it is not suitable for image data generation~\cite{fan2020survey}. 

\section*{Appendix 2: Effects of Hyperparameter Choices on Membership Privacy}
Note: For experiments comparing different number of generators, LFW has been left out since the small size of the dataset leads to partition size being smaller than batch size for $n=6,8$. For experiments comparing different number of epochs, we restrict our comparison to MNIST and fashion-MNIST due to computational considerations. 

\begin{table}[h!]
\small
\begin{center}
\begin{tabular}{ c|c|c|c|c } 
 \hline
 Dataset & $n=2$ & $n=4$ & $n=6$ & $n=8$
 \\ \hline
 MNIST  & 0.12 & 0.095 &0.098 & 0.08\\ 
 f-MNIST  & 0.194 & 0.17 & 0.14 & 0.126\\
 CIFAR-10 & 0.424 & 0.139 & 0.11 & 0.101\\
 \hline
\end{tabular}
\end{center}
 \caption{White--box attack accuracy against privGAN for different number of generator/discriminator pairs. $\lambda=1$ in all cases.}
 \label{SuppTable4}
\end{table}

\begin{table}[h!]
\small
\begin{center}
\begin{tabular}{ c|c|c|c|c } 
 \hline
 Dataset & $n=2$ & $n=4$ & $n=6$ & $n=8$
 \\ \hline
 MNIST  & 0.235 & 0.084 & 0.054 & 0.03\\ 
 f-MNIST  & 0.278 & 0.235 & 0.166 & 0.101\\
 CIFAR-10 & 0.367 & 0.247 & 0.138 & 0.056\\
 \hline
\end{tabular}
\end{center}
 \caption{TVD attack score against privGAN for different number of generator/discriminator pairs. $\lambda=1$ in all cases.}
 \label{SuppTable5}
\end{table}

\begin{table}[h!]
\small
\begin{center}
\begin{tabular}{ c|c|c|c|c } 
 \hline
 Dataset & $n=2$ & $n=4$ & $n=6$ & $n=8$
 \\ \hline
 MNIST  & 0.68 & 0.53 & 0.6 & 0.57\\ 
 f-MNIST  & 0.7 & 0.69 & 0.5 & 0.6\\
 CIFAR-10 & 0.56 & 0.55 & 0.56 & 0.52\\
 \hline
\end{tabular}
\end{center}
 \caption{Monte--Carlo attack accuracy against privGAN for different number of generator/discriminator pairs. $\lambda=1$ in all cases.}
 \label{SuppTable6}
\end{table}

\begin{figure}[H]
    \centering
    \includegraphics[scale=0.4]{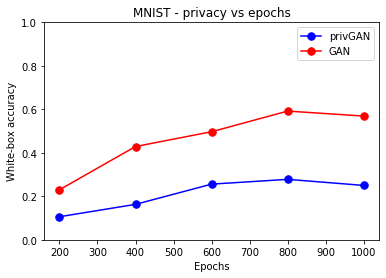}
    \caption{Comparison of white--box privacy loss between privGAN and GAN for different number of training epochs (for MNIST). $\lambda=1$  and $N=2$ for privGAN.}
    \label{priv-epochs-mnist}
\end{figure}

\begin{figure}[H]
    \centering
    \includegraphics[scale=0.4]{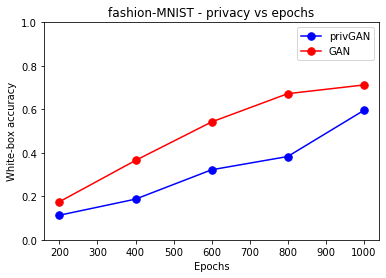}
    \caption{Comparison of white--box privacy loss between privGAN and GAN for different number of training epochs (for fashion--MNIST).  $\lambda=1$  and $N=2$ for privGAN.}
    \label{priv-epochs-fmnist}
\end{figure}

\onecolumn

\section*{Appendix 4: Complete Table of White--Box Attack Efficiencies}

\begin{table*}[h!]
\small
\begin{center}
\begin{tabular}{ c|c|c|c|c|c|c|c } 
 \hline
 Dataset & Rand. & GAN & privGAN ($\lambda=0.1$) & privGAN ($\lambda=1.0$) & privGAN ($\lambda=10.0$)  & DPGAN ($\epsilon=100$) & DPGAN ($\epsilon=25$)
 \\ \hline
 MNIST  & 0.1 & 0.467 & 0.144 & 0.12 & 0.096 & 0.098 & 0.1\\ 
 f-MNIST  & 0.1 & 0.527 & 0.192 & 0.192 & 0.095  & 0.102 & 0.099\\
 LFW & 0.1 & 0.724 & 0.148 & 0.107 & 0.086 & 0.109 & 0.097\\ 
 CIFAR-10 & 0.1 & 0.723 & 0.568 & 0.424 & 0.154 & 0.107 & 0.098\\
 \hline
\end{tabular}
\end{center}
 \caption{White box attack accuracy of various models on various datasets.}
 \label{SuppTable1}
\end{table*}

\begin{table*}[h!]
\small
\begin{center}
\begin{tabular}{ c|c|c|c|c|c|c } 
 \hline
 Dataset & GAN & privGAN ($\lambda=0.1$) & privGAN ($\lambda=1.0$) & privGAN ($\lambda=10.0$)  & DPGAN ($\epsilon=100$) & DPGAN ($\epsilon=25$)
 \\ \hline
 MNIST  & 0.438 & 0.31 & 0.235 & 0.048 & 0.021 & 0.024\\ 
 f-MNIST  & 0.674 & 0.323 & 0.278 & 0.155 & 0.022 & 0.025\\
 LFW & 0.756 & 0.237 & 0.097 & 0.261 & 0.04 & 0.045\\ 
 CIFAR-10 & 0.91 & 0.371 & 0.367 & 0.244 & 0.01 & 0.008\\
 \hline
\end{tabular}
\end{center}
 \caption{TVD attack score of various models on various datasets.}
 \label{SuppTable2}
\end{table*}

\begin{table*}[h!]
\small
\begin{center}
\begin{tabular}{ c|c|c|c|c|c|c|c } 
 \hline
 Dataset & Rand. & GAN & privGAN ($\lambda=0.1$) & privGAN ($\lambda=1.0$) & privGAN ($\lambda=10.0$)  & DPGAN ($\epsilon=100$) & DPGAN ($\epsilon=25$)
 \\ \hline
 MNIST  & 0.5 & 0.79 & 0.71 & 0.68 & 0.56 & 0.62 & 0.6\\ 
 f-MNIST  & 0.5 & 0.75 & 0.73 & 0.7  & 0.64 & 0.52 & 0.56\\
 LFW & 0.5 & 0.77 & 0.66 & 0.57 & 0.55 & 0.59 & 0.49\\ 
 CIFAR-10 & 0.5 & 0.62 & 0.61 & 0.56 & 0.52 & 0.63 & 0.57\\
 \hline
\end{tabular}
\end{center}
 \caption{Monte--Carlo attack accuracy of various models on various datasets.}
 \label{SuppTable3}
\end{table*}

\section*{Appendix 3: Privacy vs Utility Plots for PATE-GAN}
\begin{figure}[h!]
    \centering
    \includegraphics[scale=0.4]{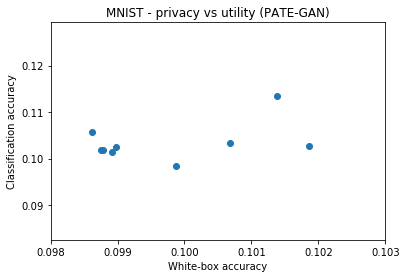}
    \caption{Utility vs privacy for PATE--GAN on the MNIST dataset.}
    \label{pate-mnist}
\end{figure}

\begin{figure}[h!]
    \centering
    \includegraphics[scale=0.4]{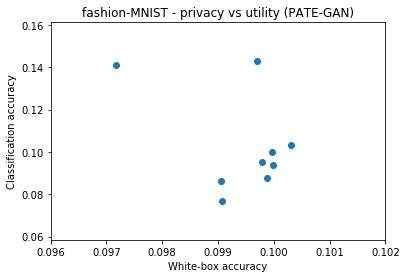}
    \caption{Utility vs privacy for PATE--GAN on the fashion--MNIST dataset.}
    \label{pate-fmnist}
\end{figure}
\end{document}